\documentclass{article}

\PassOptionsToPackage{numbers,sort, compress}{natbib}

\usepackage[final]{neurips_2021}

\usepackage[utf8]{inputenc} 
\usepackage[T1]{fontenc}    
\usepackage{hyperref}       
\usepackage{url}            
\usepackage{booktabs}       
\usepackage{amsfonts}       
\usepackage{nicefrac}       
\usepackage{microtype}      
\usepackage[dvipsnames]{xcolor}         

\usepackage{amsmath}
\usepackage{tikz}
\usepackage{wrapfig} 

\usepackage{amsmath, amssymb, amsthm, bbm, mathtools, commath, amsfonts, tikz-cd}
\usepackage{multirow}
\usepackage{caption}
\usepackage{subcaption}
\usepackage{float}

\newcommand{\mrm}[1]{\mathrm{#1}}

\newtheorem{theorem}{Theorem}

\newtheorem{lemma}{Lemma}

\newcommand{\R}{\mathbb{R}}
\newcommand{\RR}{\mathbb{R}}

\renewcommand{\norm}[1]{\left\|#1\right\|}

\newcommand{\M}{\mathcal{M}}

\newtheoremstyle{TheoremNum}
    {\topsep}{\topsep}              
    {\itshape}                      
    {}                              
    {\bfseries}                     
    {.}                             
    { }                             
    {\thmname{#1}\thmnote{ \bfseries #3}}
\theoremstyle{TheoremNum}
\newtheorem{thmn}{Theorem}

\newtheoremstyle{CorollaryNum}
    {\topsep}{\topsep}              
    {\itshape}                      
    {}                              
    {\bfseries}                     
    {.}                             
    { }                             
    {\thmname{#1}\thmnote{ \bfseries #3}}
\theoremstyle{CorollaryNum}

\newtheoremstyle{LemmaNum}
    {\topsep}{\topsep}              
    {\itshape}                      
    {}                              
    {\bfseries}                     
    {.}                             
    { }                             
    {\thmname{#1}\thmnote{ \bfseries #3}}
\theoremstyle{LemmaNum}

\usepackage{algorithm}

\title{Equivariant Manifold Flows}

\author{
    Isay Katsman*, Aaron Lou*, Derek Lim*, Qingxuan Jiang* \\
    Cornell University\\
    \texttt{\{isk22, al968, dl772, qj46\}@cornell.edu}\\
    \AND
    Ser-Nam Lim\\
    Facebook AI\\
    \texttt{sernam@gmail.com}\\
    \And
    Christopher De Sa\\
    Cornell University\\
    \texttt{cdesa@cs.cornell.edu}
}

\begin{document}

\maketitle

\begin{abstract}
    Tractably modelling distributions over manifolds has long been an important goal in the natural sciences. Recent work has focused on developing general machine learning models to learn such distributions. However, for many applications these distributions must respect manifold symmetries---a trait which most previous models disregard.
In this paper, we lay the theoretical foundations for learning symmetry-invariant distributions on arbitrary manifolds via equivariant manifold flows. We demonstrate the utility of our approach by learning quantum field theory-motivated invariant $SU(n)$ densities and by correcting meteor impact dataset bias.

\end{abstract}

\section{Introduction}
\label{sec:intro}


\begin{wrapfigure}{R}{.35\textwidth}
  \vspace{-5pt}
  \includegraphics[width=.35\textwidth]{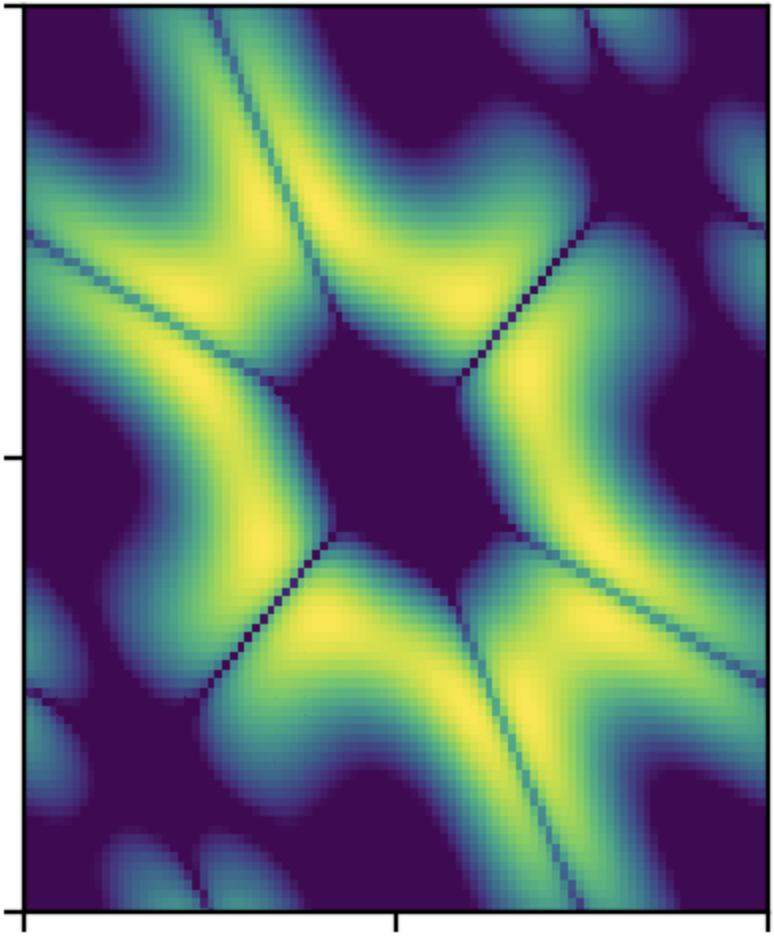}
  \caption{An example of a density on $SU(3)$ that is invariant to conjugation by $SU(3)$. The $x$-axis and $y$-axis are the angles $\theta_1$ and $\theta_2$ for eigenvalues $e^{i\theta_1}$ and $e^{i\theta_2}$ of a matrix in $SU(3)$. The axis range is $-\pi$ to $\pi$.}
  \label{fig:example_su3}
  \vspace{-40pt}
\end{wrapfigure}
{\let\thefootnote\relax\footnotetext{* indicates equal contribution}}

Learning probabilistic models for data has long been the focus of many problems in machine learning and statistics. Though much effort has gone into learning models over Euclidean space \citep{Goodfellow2014GenerativeAN, chen2018neural, grathwohl2018scalable}, less attention has been allocated to learning models over non-Euclidean spaces, despite the fact that many problems require a manifold structure. Density learning over non-Euclidean spaces has applications ranging from quantum field theory in physics~\citep{Wirnsberger2020TargetedFE} to motion estimation in robotics~\citep{Feiten2013RigidME} to protein-structure prediction in computational biology~\citep{Hamelryck2006SamplingRP}.

Continuous normalizing flows (CNFs) \citep{chen2018neural, grathwohl2018scalable} are powerful generative models for learning structure in complex data due to their tractability and theoretical guarantees. Recent work \citep{lou2020neural, mathieu2020riemannian} has extended the framework of continuous normalizing flows to the setting of density learning on Riemannian manifolds. 
However, for many applications in the natural sciences, this construction is insufficient as it cannot properly model necessary symmetries.
For example, such symmetry requirements arise when sampling coupled particle systems in physical chemistry \citep{kohler2020equivariant} or sampling for use in $SU(n)$\footnote{$SU(n)$ denotes the special unitary group $SU(n) = \{ X \in \mathbb{C}^{n \times n} \mid X^* X = I, \; \det(X) = 1\}$.} lattice gauge theories in theoretical physics \citep{boyda2020sampling}.

More precisely, these symmetries are invariances with respect to action by an isometry subgroup of the underlying manifold. For example, consider the task of learning a density on the sphere that is invariant to rotation around an axis; this is an example of learning an isometry subgroup invariant\footnote{This specific isometry subgroup is known as the isotropy group at a point of the sphere intersecting the axis.} density. For a less trivial example, note that when learning a flow-based sampler for $SU(n)$ in the context of lattice QFT \cite{boyda2020sampling}, the learned density must be invariant to conjugation by $SU(n)$ (see Figure \ref{fig:example_su3} for a density on $SU(3)$ that exhibits the requisite symmetry).


One might naturally attempt to work with the quotient of the manifold by the relevant isometry subgroup in order to model the invariance. First, note that this structure is not always a manifold, and additional restrictions are needed on the action to ensure the quotient will have a manifold structure\footnote{In particular, the isometry subgroup action needs to be smooth, free, and proper to ensure the quotient will be a manifold by the Quotient Manifold Theorem \cite{lee2013introduction}.}. Assuming the quotient is in fact a manifold, one then asks whether an invariant density may be modelled by learning over this quotient with a general manifold density learning method such as NMODE \cite{lou2020neural}? Though this seems plausible, it is a problematic approach for several reasons:
\begin{enumerate}
    \item First, it is often difficult to realize necessary constructs (charts, exponential maps, tangent spaces) on the quotient manifold (e.g. this is the case for $\mathbb{RP}^n$, a quotient of $\mathbb{S}^n$ \cite{lee2013introduction}).
    
    
    \item Second, even if the above constructs can be realized, the quotient manifold often has a boundary, which precludes the use of a manifold CNF. To illustrate this point, consider the simple case of the sphere invariant to rotation about an axis; the quotient manifold is a closed interval, and a CNF would ``flow out" on the boundary.
    
    
    \item Third, even if the quotient is a manifold without boundary for which we have a clear characterization, it may have a discrete structure that induces artifacts in the learned distribution. This is the case for \citet{boyda2020sampling}: the flow construction over the quotient induces abnormalities in the density.
    
\end{enumerate}
Motivated by the above drawbacks, we design a manifold continuous normalizing flow on the original manifold that maintains the requisite symmetry invariance. Since vanilla manifold CNFs do not maintain said symmetries, we instead construct \emph{equivariant} manifold flows and show they induce the desired invariance. To construct these flows, we present the first general way of designing equivariant vector fields on manifolds.
A summary of our paper's contributions is as follows:


\begin{itemize}
    \item We present a general framework and the requisite theory for learning equivariant manifold flows: in our setup, the flows can be learned over arbitrary Riemannian manifolds while explicitly incorporating symmetries inherent to the problem. Moreover, we prove that the equivariant flows we construct can universally approximate distributions on closed manifolds.
    \item We demonstrate the efficacy of our approach by learning gauge invariant densities over $SU(n)$ in the context of quantum field theory. In particular, when applied to the densities in \citet{boyda2020sampling}, we adhere more naturally to the target geometry and avoid the unnatural artifacts of the quotient construction.
    \item We highlight the benefit of incorporating symmetries into manifold flow models by comparing directly against previous general manifold density learning approaches. We show that when a general manifold learning model is not aware of symmetries inherent to the problem, the learned density is of considerably worse quality and violates said symmetries. Prior to our work, there did not exist literature that demonstrated the benefits of incorporating isometry group symmetries for learning flows on manifolds, yet we achieve these benefits, and do so through a novel equivariant vector field construction.

\end{itemize}

\section{Related Work}
\label{sec:rel_work}

Our work builds directly on pre-existing manifold normalizing flow models and enables them to leverage inherent symmetries through equivariance. In this section we cover important developments from the relevant fields: manifold normalizing flows and equivariant machine learning.

\paragraph{Normalizing Flows on Manifolds} Normalizing flows on Euclidean space have long been touted as powerful generative models \citep{dinh2016density, chen2018neural, grathwohl2018scalable}. Similar to GANs \cite{Goodfellow2014GenerativeAN} and VAEs \cite{kingma2013auto}, normalizing flows learn to map samples from a tractable prior density to a target density. However, unlike the aforementioned models, normalizing flows account for changes in volume, enabling exact evaluation of the output probability density. In a rather concrete sense, this makes them theoretically principled. As such, they are ideal candidates for generalization beyond the Euclidean setting, where a careful, theoretically principled modelling approach is necessary.

Motivated by recent developments in geometric deep learning \cite{Bronstein2021GeometricDL}, many methods have extended normalizing flows to Riemannian manifolds. \citet{rezende2020normalizing} introduced constructions specific to tori and spheres, while \citet{bose2020latent} introduced constructions for hyperbolic space. Following this work, \citet{lou2020neural, mathieu2020riemannian, Falorsi2020NeuralOD} concurrently introduced a general construction by extending Neural ODEs \citep{chen2018neural} to the setting of Riemannian manifolds. Our work takes inspiration from the methods of \citet{lou2020neural, mathieu2020riemannian} and generalizes them further to enable learning that takes into account symmetries of the target density.

\paragraph{Equivariant Machine Learning} Motivated by the observation that many classic neural network architectures incorporate symmetry as an inductive bias, recent work has leveraged symmetries inherent in data through the concept of equivariance \citep{cohen2016group, cohen2018spherical, cohen2019gauge, kondor2018generalization, finzi2020generalizing, rezende2019equivariant}. \citet{kohler2020equivariant}, in particular, used equivariant normalizing flows to enable learning symmetric densities over Euclidean space. The authors note their approach is better suited to density learning in some physical chemistry settings (when compared to general purpose normalizing flows), since they take into account the symmetries of the problem.

Symmetries also appear naturally in the context of learning densities over manifolds. While in many cases symmetry can be a good inductive bias for learning\footnote{For example, asteroid impacts on the sphere can be modelled as being approximately invariant to rotation about the Earth's axis.}, for certain test tasks it is a strict requirement. For example, \citet{boyda2020sampling} introduced equivariant flows on $SU(n)$ for use in lattice gauge theories, where the modelled distribution must be conjugation invariant. However, beyond conjugation invariant learning on $SU(n)$ \cite{boyda2020sampling}, not much other work has been done for learning invariant distributions over manifolds. Our work bridges this gap by introducing the first general equivariant manifold normalizing flow model for arbitrary manifolds and symmetries.



\vspace{-5pt}
\section{Background}
\label{sec:background}

In this section, we provide a terse overview of necessary concepts for understanding our paper. In particular, we address fundamental notions from Riemannian geometry as well as the basic set-up of normalizing flows on manifolds. For a more detailed introduction to Riemannian geometry, we refer the reader to textbooks such as \citet{lee2013introduction} and \citet{Kobyzev_2020}.

\vspace{-5pt}

\subsection{Riemannian Geometry}

A Riemannian manifold $(\M, h)$ is an $n$-dimensional manifold with a smooth collection of inner products $(h_x)_{x \in \M}$ for every tangent space $T_x\M$. The Riemannian metric $h$ induces a distance $d_h$ on the manifold.

A diffeomorphism $f: \M \to \M$ is a differentiable bijection with differentiable inverse. A diffeomorphism $f: \M \to \M$ is called an isometry if $h(D_xf(u), D_xf(v)) = h(u, v)$ for all tangent vectors $u, v \in T_x\M$ where $D_xf$ is the differential of $f$. Note that isometries preserve the manifold distance function. The collection of all isometries forms a group $G$, which we call the isometry group of the manifold $\M$.

Riemannian metrics also allow for a natural analogue of gradients on $\R^n$. For a function $f: \M \to \R$, we define the Riemannian gradient $\nabla_x f$ to be the vector on $T_x \M$ such that $h(\nabla_xf, v) = D_xf(v)$ for $v \in T_x\M$.

\subsection{Normalizing Flows on Manifolds}

\paragraph{Manifold Normalizing Flow} Let $(\mathcal{M}, h)$ be a Riemannian manifold. A normalizing flow on $\mathcal{M}$ is a diffeomorphism $f_{\theta}:\mathcal{M}\rightarrow\mathcal{M}$ (parametrized by $\theta$) that transforms a prior density $\rho$ to model density $\rho_{f_{\theta}}$. The model distribution can be computed via the Riemannian change of variables\footnote{Here, $\det_h$ is the determinant function with volume induced by the Riemannian metric $h$.}: \begin{equation*}
    \rho_{f_{\theta}}(x)=\rho\left(f_{\theta}^{-1}(x)\right)\left|{\textstyle\det_h} D_x f_{\theta}^{-1}\right|.
\end{equation*}

\paragraph{Manifold Continuous Normalizing Flow} A manifold continuous normalizing flow with base point $z$ is a function $\gamma: [0, \infty) \to \M$ that satisfies the manifold ODE \begin{equation*}
    \frac{d\gamma(t)}{dt} = X(\gamma(t), t)\text{,} \qquad \gamma(0)=z.
\end{equation*} 

We define $F_{X, T}: \M \to \M$, $z\mapsto F_{X, T}(z)$ to map any base point $z\in \M$ to the value of the CNF starting at $z$, evaluated at time $T$. This function is known as the (vector field) flow of $X$.

\subsection{Equivariance and Invariance}

Let $G$ be an isometry subgroup of $\M$. We notate the action of an element $g \in G$ on $\M$ by the map $L_g : \M \rightarrow \M$.

\textbf{Equivariant and Invariant Functions}\hspace{1em} We say that a function $f: \mathcal{M} \to \mathcal{N}$ is equivariant if, for all isometries $g_x: \mathcal{M} \to \mathcal{M}$ and $g_y: \mathcal{N} \to \mathcal{N}$, we have $f \circ g_x = g_y \circ f$. We say a function $f: \mathcal{M} \to \mathcal{N}$ is invariant if $f \circ g_x = f$.

\textbf{Equivariant Vector Fields}\hspace{1em}  Let $X: \M \times [0, \infty) \to T\M$, $X(m, t) \in T_m\M$ be a time-dependent vector field on manifold $\mathcal{M}$, with base point $x_{0}\in \M$. $X$ is a $G$-equivariant vector field if $\forall (m, t)\in \mathcal{M}\times [0, \infty)$, $X(L_g m, t) = (D_m L_g) X(m, t)$.

\textbf{Equivariant Flows}\hspace{1em}  A flow $f:\M\rightarrow\M$ is $G$-equivariant if it commutes with actions from $G$, i.e. we have $L_g \circ f = f \circ L_g$.

\textbf{Invariance of Density}\hspace{1em}  A density $\rho$ on a manifold $\M$ is $G$-invariant if, for all $g \in G$ and $x \in \M$ , $\rho(L_gx) = \rho(x)$, where $L_g$ is the action of $g$ on $x$.

\section{Invariant Densities from Equivariant Flows}
\label{sec:theory_main}



Our goal in this section is to describe a tractable way to learn a density over a manifold that obeys a symmetry given by an isometry subgroup $G$. Since this cannot be done directly and it is not clear how a manifold continuous normalizing flow can be altered to preserve symmetry, we will derive the following implications to yield a tractable solution:

\begin{enumerate}
    \item \textbf{$G$-invariant potential $\Rightarrow$ $G$-equivariant vector field (Theorem \ref{thm:1}).} We show that given a $G$-invariant potential function $\Phi: \M \rightarrow \mathbb{R}$, the vector field $\nabla \Phi$ is $G$-equivariant.
    \item \textbf{$G$-equivariant vector field $\Rightarrow$ $G$-equivariant flow (Theorem \ref{thm:2}).} We show that a $G$-equivariant vector field on $\M$ uniquely induces a $G$-equivariant flow.
    \item \textbf{$G$-equivariant flow $\Rightarrow$ $G$-invariant density (Theorem \ref{thm:3}).} We show that given a $G$-invariant prior $\rho$ and a $G$-equivariant flow $f$, the flow density $\rho_{f}$ is $G$-invariant.
\end{enumerate}

These are constructed in the same spirit as the theorems in \citet{kohler2020equivariant} (which also appeared in \citet{papamakarios2019normalizing}), although we note that our results are significantly more general. In addition to extending the domain to Riemannian manifolds, we consider arbitrary symmetry groups while \citet{kohler2020equivariant} only considers the linear Lie group $SO(n)$. As a result, our proof techniques are based on heavy geometric machinery instead of straightforward linear algebra techniques.

If we have a prior distribution on the manifold that obeys the requisite invariance, then the above implications show that we can use a $G$-invariant potential to produce a flow that, in tandem with the CNF framework, learns an output density with the desired invariance. We claim that constructing a $G$-invariant potential function on a manifold is far simpler than directly parameterizing a $G$-invariant density or a $G$-equivariant flow. We shall give explicit examples of $G$-invariant potential constructions in Section \ref{sec:constructing_potentials} that induce a desired density invariance.

Moreover, we show in Theorem \ref{thm:4} that considering equivariant flows generated from invariant potential functions suffices to learn any smooth distribution over a closed manifold, as measured by Kullback-Leibler divergence.

We defer the proofs of all theorems to the appendix.

\subsection{Equivariant Gradient of Potential Function}

We start by showing how to construct $G$-equivariant vector fields from $G$-invariant potential functions.

To design an equivariant vector field $X$, it is sufficient to set the vector field dynamics of $X$ as the gradient of some $G$-invariant potential function $\Phi:\mathcal{M}\rightarrow\mathbb{R}$. This is formalized in the following theorem.

\begin{theorem}\label{thm:1}
Let $(\M, h)$ be a Riemannian manifold and $G$ be its group of isometries (or an isometry subgroup). If $\Phi: \M \to \R$ is a smooth $G$-invariant function, then the following diagram commutes for any $g \in G$: 

\begin{center}
    \begin{tikzcd}
        \M \arrow[r, "L_g"] \arrow[d, "\nabla \Phi"]& \M \arrow[d, "\nabla \Phi"]\\
        T\M \arrow[r, "DL_g"] & T\M
    \end{tikzcd}
\end{center}

or $\nabla_{L_g u} \Phi = D_u L_g(\nabla_u \Phi)$. Hence $\nabla \Phi$ is a $G$-equivariant vector field. This condition is also tight in the sense that it only occurs if $G$ is the isometry subgroup.
\end{theorem}

Hence, as long as one can construct a $G$-invariant potential function, one can obtain the desired equivariant vector field. By this construction, a parameterization of $G$-invariant potential functions yields a parameterization of (some) $G$-equivariant vector fields.

\subsection{Constructing Equivariant Manifold Flows from Equivariant Vector Fields}

To construct equivariant manifold flows, we will use tools from the theory of manifold ODEs. In particular, there exists a natural correspondence between equivariant flows and equivariant vector fields. We formalize this in the following theorem:

\begin{theorem}\label{thm:2}
Let $(\M, h)$ be a Riemannian manifold, and $G$ be its isometry group (or one of its subgroups). Let $X$ be any time-dependent vector field on $\mathcal{M}$, and $F_{X, T}$ be the flow of $X$. Then $X$ is a $G$-equivariant vector field if and only if $F_{X,T}$ is a $G$-equivariant vector field flow.
\end{theorem}

Hence we can obtain an equivariant flow from an equivariant vector field, and vice versa.

\subsection{Invariant Manifold Densities from Equivariant Flows}

We now show that $G$-equivariant flows induce $G$-invariant densities. Note that we require the group $G$ to be an isometry subgroup in order to control the density of $\rho_f$, and the following theorem does not hold for general diffeomorphism subgroups.

\begin{theorem}\label{thm:3}
Let $(\M, h)$ be a Riemannian manifold, and $G$ be its isometry group (or one of its subgroups). If $\rho$ is a $G$-invariant density on $\M$, and $f$ is a $G$-equivariant diffeomorphism, then $\rho_f$ is also $G$-invariant.
\end{theorem}

In the context of manifold normalizing flows, Theorem \ref{thm:3} implies that if the prior density on $\M$ is $G$-invariant and the flow is $G$-equivariant, the resulting output density will be $G$-invariant. In the context of the overall set-up, this reduces the problem of constructing a $G$-invariant density to the problem of constructing a $G$-invariant potential function.

\subsection{Sufficiency of Flows Generated via Invariant Potentials}

It is unclear whether equivariant flows induced by invariant potentials can learn arbitrary invariant distributions over manifolds. In particular, it is reasonable to have some concerns about limited expressivity, since it is unclear whether any equivariant flow can be generated in this way. We alleviate these concerns for our use cases by proving that equivariant flows obtained from invariant potential functions suffice to learn any smooth invariant distribution over a closed manifold, as measured by Kullback-Leibler (KL) divergence.

\begin{theorem}\label{thm:4}
Let $(\M, h)$ be a closed Riemannian manifold. Let $\pi$ be a smooth, non-vanishing distribution over $\M$, which will act as our target distribution. Let $\rho_t$ be a distribution over said manifold parameterized by a real time variable $t$, with $\rho_0$ acting as the initial distribution. Let $D_{KL}(\rho_t||\pi)$ denote the Kullback–Leibler divergence between distributions $\rho_t$ and $\pi$. If we choose a $g: \M \rightarrow \mathbb{R}$ such that
\[
g(x) = \log \left( \frac{\pi(x)}{\rho_t(x)} \right),
\]
and if $\rho_t$ evolves with $t$ as the distribution of a flow according to $g$, it follows that
\[
\frac{\partial}{\partial t} D_{KL} (\rho_t \| \pi) = - \int_{\mathcal{M}} \rho_t \exp(g) \| \nabla g \|^2\: dx
= - \int_\mathcal{M} \pi \| \nabla g \|^2\: dx
\]

implying convergence of $\rho_t$ to $\pi$ in $KL$. Moreover, the exact diffeomorphism that takes us from $\rho_0 \rightarrow \pi$ is as follows. Given some initial point $x \in \mathcal{M}$, let $u(t)$ be the solution to the initial value problem given by:
\[
\frac{du(t)}{dt} = \nabla g(t), \qquad u(0) = x
\]
The desired diffeomorphism maps $x$ to $\lim_{t \rightarrow \infty} u(t)$.
\end{theorem}


Hence if the target distribution is $\pi$, the current distribution is $\rho_0$, and $g$ as defined above is the potential from which the flow controlling the evolution of $\rho_t$ is obtained, then $\rho_t$ converges to $\pi$ in $KL$. This means that considering flows generated by invariant potential functions is sufficient to learn any smooth invariant target distribution on a closed manifold (as measured by KL divergence).



\section{Learning Invariant Densities with Equivariant Flows}

In this section, we discuss implementation details of the methodology given in Section \ref{sec:theory_main}. In particular, we describe the equivariant manifold flow model, provide two examples of invariant potential constructions on different manifolds, and discuss how training is performed depending on the target~task.

\subsection{Equivariant Manifold Flow Model}
\label{sec:equiv_manifold_flow}

For our equivariant flow model, we first construct a $G$-invariant potential function $\Phi: \M \rightarrow \mathbb{R}$ (we show how to construct these potentials in Section~\ref{sec:constructing_potentials}). The equivariant flow model works by using automatic differentiation \cite{Paszke2017AutomaticDI} on $\Phi$ to obtain $\nabla \Phi$, using this $\nabla \Phi$ for the vector field, and integrating in a step-wise fashion over the manifold. Specifically, forward integration and change-in-density (divergence) computations utilize the Riemannian Continuous Normalizing Flows \citep{mathieu2020riemannian} framework. This flow model is used in tandem with a specific training procedure (described in Section \ref{sec:training_paradigm}) to obtain a $G$-invariant model density that approximates some target.

\subsection{Constructing \texorpdfstring{$G$-invariant}{G-invariant} Potential Functions}
\label{sec:constructing_potentials}

In this subsection, we present two constructions of invariant potentials on manifolds. Note that a symmetry of a manifold (i.e. action by an isometry subgroup) will leave part of the manifold free. The core idea of our invariant potential construction is to parameterize a neural network on the free portion of the manifold. While the two constructions we give below are certainly not exhaustive, they illustrate the versatility of our method, which is applicable to general manifolds and symmetries.

\subsubsection{Isotropy Invariance on \texorpdfstring{$S^{2}$}{S2}}
\label{sec:sphere_model}

Consider the sphere $S^2$, which is the Riemannian manifold $\{v \in \R^3: \norm{v} = 1\}$ with the induced pullback metric. The isotropy group for a point $v$ is defined as the subgroup of the isometry group which fixes $v$, i.e. the set of rotations around an axis that passes through $v$. In practice, we let $v = (0, 0, 1)$, so the isotropy group is the group of rotations on the $xy$-plane. An isotropy invariant density would be invariant to such rotations, and hence would look like a horizontally-striped density on the sphere (see Figure \ref{fig:sphere_invariant}).

\paragraph{Invariant Potential Parameterization} We design an invariant potential by applying a neural network to the free parameter. In the case of our specific isotropy group listed above, the free parameter is the $z$-coordinate. The invariant potential is simply a $2$-input neural network with the spatial input being the $z$-coordinate and the time input being the time during integration. As a result of this design, we see that the only variance in the learned distribution that uses this potential will be along the $z$-axis, as desired.

\paragraph{Prior Distributions} For proper learning with a normalizing flow, we need a prior distribution on the sphere that respects the isotropy invariance. There are many isotropy invariant potentials on the sphere. Natural choices include the uniform density (which is invariant to all rotations) and the wrapped distribution with the center at $v$ \citep{Skopek2020MixedcurvatureVA, Nagano2019AWN}. For our experiments, we use the uniform density.

\subsubsection{Conjugation Invariance on \texorpdfstring{$SU(n)$}{SU(n)}}\label{sec:sun_model}

For many applications in physics (specifically gauge theory and lattice quantum field theory), one works with the Lie Group $SU(n)$ --- the group of unitary matrices with determinant $1$. In particular, when modelling probability distributions on $SU(n)$ for lattice QFT, the desired distribution must be invariant under conjugation by $SU(n)$ \citep{boyda2020sampling}. Conjugation is an isometry on $SU(n)$ (see Appendix~\ref{appendix:conjugation_isometry}), so we can model probability distributions invariant under this action with our developed theory.

\paragraph{Invariant Potential Parameterization}\label{sec:sun_invariant}

We want to construct a conjugation invariant potential function $\Phi: SU(n)\to\RR$. Note that matrix conjugation preserves eigenvalues. Thus, for a function $\Phi: SU(n) \to \RR$ to be invariant to matrix conjugation, it has to act on the eigenvalues of $x \in SU(n)$ as a multi-set.

We can parameterize such potential functions $\Phi$ by the DeepSet network from \cite{zaheer2017deep}. DeepSet is a permutation invariant neural network that acts on the eigenvalues, so the mapping of $x \in SU(n)$ is $\Phi(x) = \hat \Phi(\{\lambda_1(x), \ldots, \lambda_n(x) \})$ for some set function $\hat \Phi$. We append the integration time to the input of the standard neural network layers in the DeepSet network.

As a result of this design, we see that the only variance in the learned distribution will be amongst non-similar matrices, while all similar matrices will be assigned the same density value.

\paragraph{Prior Distributions}

For the prior distribution of the flow, we need a distribution that respects the matrix conjugation invariance. We use the Haar measure on $SU(n)$, which is the uniform density over this manifold that is symmetric under gauge symmetry \citep{boyda2020sampling}. The volume element of the Haar measure is given for an $x \in SU(n)$ as $\mrm{Haar}(x) = \prod_{i < j} |\lambda_i(x) - \lambda_j(x)|^2$. We can sample from and compute the log probabilities with respect to this distribution efficiently with standard matrix computations \citep{mezzadri2006generate}.

\subsection{Training Paradigms for Equivariant Manifold Flows}
\label{sec:training_paradigm}

There are two notable ways in which we can use the model described in Section \ref{sec:equiv_manifold_flow}. Namely, we can use it to learn to sample from a distribution for which we have a density function, or we can use it to learn the density given a way to sample from the distribution. These training paradigms are useful in different contexts, as we will see in Section \ref{sec:experiments}.

\paragraph{Learning to sample given an exact density.} In certain settings, we are given an exact density and the task is to learn a tractable sampler for the distribution. For example in \citet{boyda2020sampling}, we are given conjugation-invariant densities on $SU(n)$ for which we know the exact density function (without knowledge of any normalizing constants). In contrast to procedures for normalizing flow training that use negative log-likelihood based losses, we do not have access to samples from the target distribution. Instead, we train our models by sampling from the Haar distribution on $SU(n)$, computing the KL divergence between the probabilities that our model assigns to these samples and the probabilities of the target distribution evaluated at these samples, and backpropagating from this KL divergence loss. When this loss is minimized, we can sample from the target distribution by sampling the prior, then forwarding the prior samples through our model. In the context of \citet{boyda2020sampling}, such a flow-based sampler is important for modelling gauge theories.

\paragraph{Learning the density given a sampler.} In other settings, we are given a way to sample from a target distribution and want to learn the precise density for downstream tasks. For this setting, we sample the target distribution, use our flow to map it to a tractable prior, and use a negative log-likelihood-based loss. 
The flow will eventually learn to assign higher probabilities in sampled regions, and in doing so, will learn to approximate the target density.

\section{Experiments}
\label{sec:experiments}

\begin{figure}[t]
    \centering
    \subfloat[$SU(2)$ learned densities from (Left) our model and (Right) the \citet{boyda2020sampling} model. The target densities are in orange, while the model densities are in blue. The $x$-axis is $\theta$ for an eigenvalue $e^{i\theta}$ of a matrix in $SU(2)$ (note the second eigenvalue is determined as $e^{-i\theta}$). Our model has much better behavior in low-density regions (\citet{boyda2020sampling} fails to eliminate mass around $\pm \pi$) and more smoothly captures the targets in high-density regions.
    ]{\includegraphics[width=.455\columnwidth]{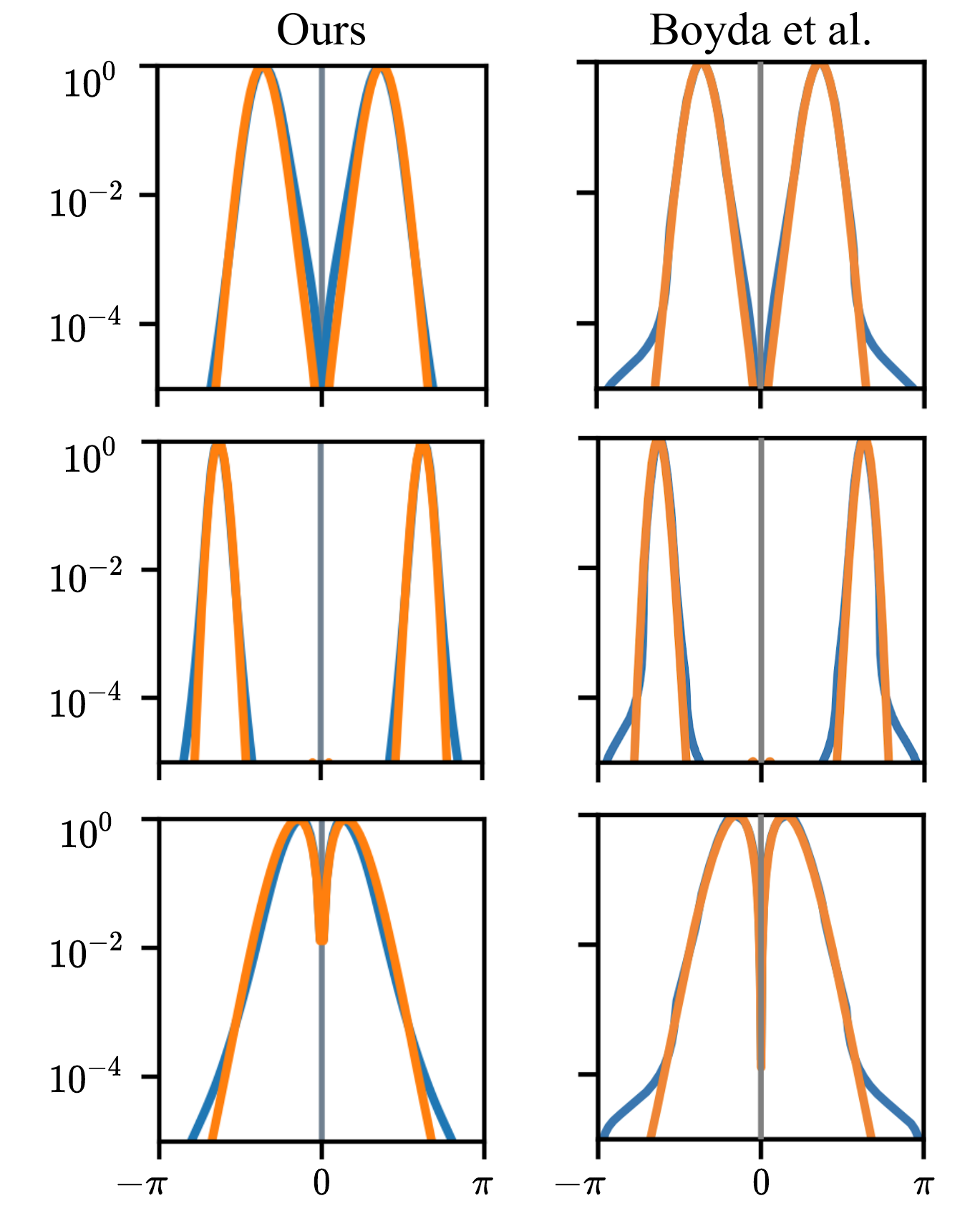}\label{fig:su2_densities}}
    \hfill
    \subfloat[$SU(3)$ learned densities from (Middle) our model and (Right) the \citet{boyda2020sampling} model for different target densities (Left).  The $x$-axis and $y$-axis are the angles $\theta_1$ and $\theta_2$ for eigenvalues $e^{i\theta_1}$ and $e^{i\theta_2}$ of a matrix in $SU(3)$ (note the third eigenvalue is determined as $e^{-i\theta_1 -i\theta_2}$), and the probabilities correspond to colors on a logarithmic scale. Our model better captures the geometry of the target densities and does not exhibit the discrete artifacts of the \citet{boyda2020sampling} model.
    ]{\includegraphics[width=.505\columnwidth]{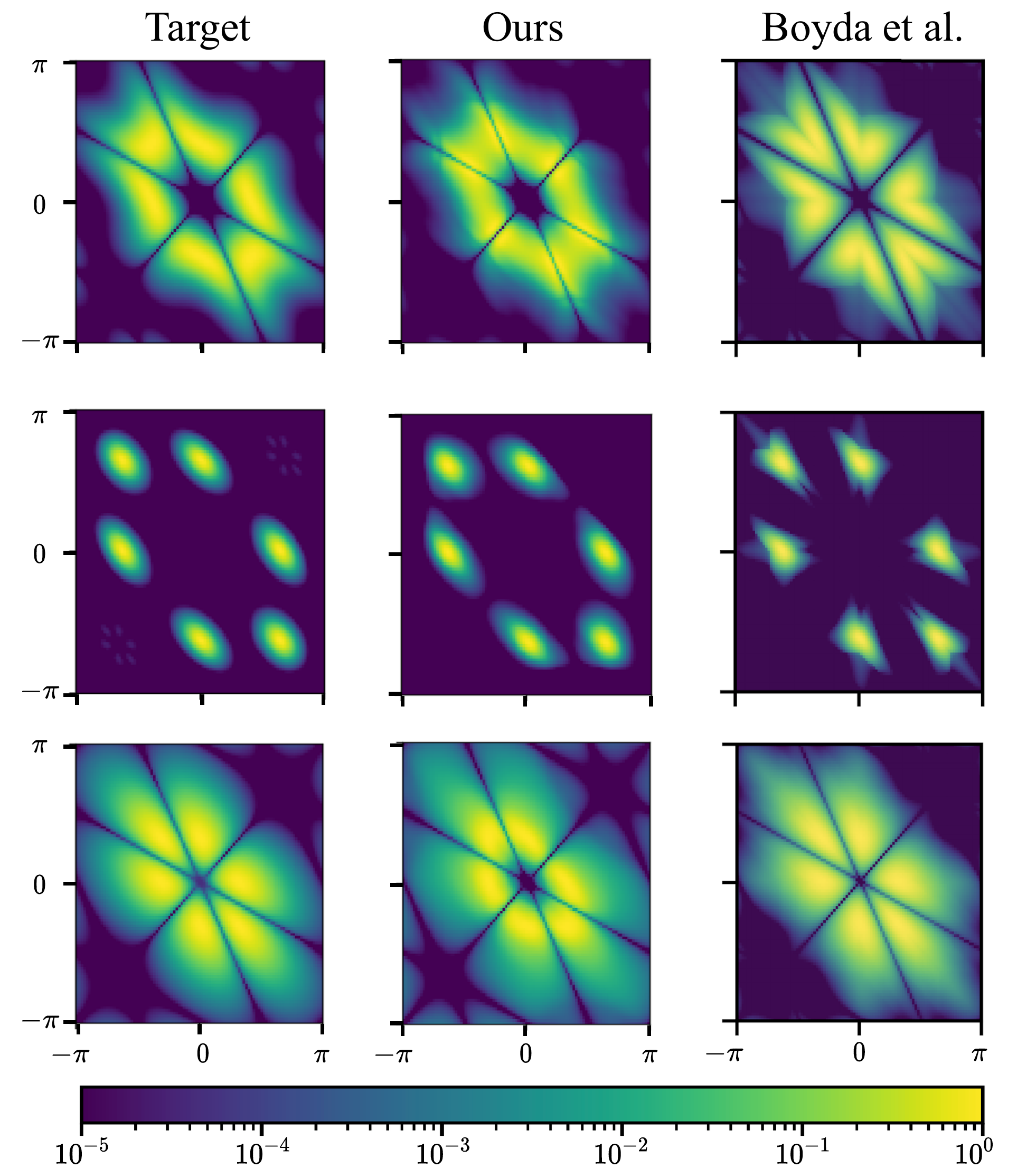}\label{fig:su3_densities}}
    \caption{ Comparison of learned densities on (a) $SU(2)$ and (b) $SU(3)$.
    All densities are normalized to have maximum value 1.}
    \label{fig:sun_densities}
    \vspace{-5pt}
\end{figure}



In this section, we utilize instantiations of equivariant manifold flows to learn densities over various manifolds of interest that are invariant to certain symmetries. First, we construct flows on $SU(n)$ that are invariant to conjugation by $SU(n)$; these are useful for lattice quantum field theory \cite{boyda2020sampling}. In this setting, our model outperforms the construction of \citet{boyda2020sampling}.

As a second application, we model asteroid impacts on Earth by constructing flow models on $S^2$ that are invariant to the isotropy group that fixes the north pole. Our approach is able to overcome dataset bias, as only land impacts are reported in the dataset.

Finally, to demonstrate the need for enforcing equivariance of flow models, we directly compare our flow construction with a general purpose flow while learning a density with an inherent symmetry. The densities we decided to use for this purpose are sphere densities that are invariant to action by the isotropy group. Our model is able to learn these densities much better than previous manifold ODE models that do not enforce equivariance of flows \citep{lou2020neural}, thus showing the ability of our model to leverage the desired symmetries. In fact, even on simple isotropy-invariant densities, our model succeeds while the free model without equivariance fails.


\subsection{\texorpdfstring{$SU(n)$}{SU(n)} Gauge Equivariant Neural Network Flows}

Learning $SU(n)$ gauge equivariant neural network flows is important for obtaining good flow-based samplers of densities on $SU(n)$ useful for lattice quantum field theory \citep{boyda2020sampling}. We compare our model for $SU(n)$ gauge equivariant flows (Section \ref{sec:sun_model}) with that of \citet{boyda2020sampling}. For the sake of staying true to the application area, we follow the framework of \citet{boyda2020sampling} in learning densities on $SU(n)$ that are invariant to conjugation by $SU(n)$. In particular, our goal is to learn a flow to model a target distribution so that we may efficiently sample from it.

As mentioned above in Section \ref{sec:training_paradigm}, this setting follows the first paradigm in which we are given exact density functions and learn how to sample.


For the actual architecture of our equivariant manifold flows, we parameterize our potentials as DeepSet networks on eigenvalues as detailed in Section~\ref{sec:sun_invariant}. The prior distribution for our model is also the Haar (uniform) distribution on $SU(n)$. Further training details are given in Appendix~\ref{appendix:training}.


\subsubsection{\texorpdfstring{$SU(2)$}{SU(2)} }

Figure~\ref{fig:su2_densities} displays learned densities for our model and the model of \citet{boyda2020sampling} in the case of three particular densities on $SU(2)$ described in Appendix~\ref{appendix:target}. While both models match the target distributions well in high-density regions, we find that our model exhibits a considerable improvement in lower-density regions, where the tails of our learned distribution decay faster. By contrast, the model of \citet{boyda2020sampling} seems to be unable to reduce mass near $\pm \pi$, a possible consequence of their construction. Even in high-density regions, our model appears to vary smoothly, with fewer unnecessary bumps and curves when compared to the densities of the model in \citet{boyda2020sampling}.

\subsubsection{\texorpdfstring{$SU(3)$}{SU(3)} }

Figure~\ref{fig:su3_densities} displays learned densities for our model and the model of \citet{boyda2020sampling} in the case of three particular densities on $SU(3)$ described in Appendix~\ref{appendix:target_su3}. In this case, we see that our models fit the target densities more accurately and better respect the geometry of the target distribution. Indeed, while the learned densities of \citet{boyda2020sampling} are often sharp and have pointed corners, our models learn densities that vary smoothly and curve in ways that are representative of the target distributions.



\begin{figure}[htb!]
    \centering
    \includegraphics[width=0.75\textwidth]{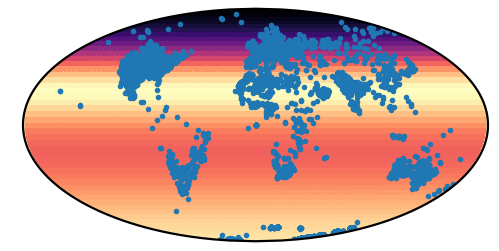}
    \caption{Our modelled distribution of meteor impacts in \citet{data.world_2017}. The true impacts are marked in blue and our isotropy invariant density is shown in the background. Note that a regular manifold normalizing flow would instead model impacts only on land as the dataset does not include any ocean impacts.}
    \label{fig:meteor}
    \vspace{-5pt}
\end{figure}

\subsection{Asteroid Impact Dataset Bias Correction}

We also showcase our model's ability to correct for dataset bias. In particular, we consider the test case of modelling asteroid impacts on Earth. Towards this end, many preexisting works have compiled locations of previous asteroid impacts \cite{data.world_2017, earth_impact_database}, but modelling these datasets is challenging since they are inherently biased. In particular, all recorded impacts are found on land. However, ocean impacts are also dangerous \cite{wardasteroidimpact} and should be properly modelled. To correct for this bias, we note that the distribution of asteroid impacts should be invariant with respect to the rotation of the Earth. We apply our isotropy invariant $S^2$ flow (described in Section \ref{sec:sphere_model}) to model the asteroid impact locations given by the dataset \citet{data.world_2017} \footnote{This dataset was released by NASA without a specified license.}. Training happens in the setting of the second paradigm described in Section \ref{sec:training_paradigm}, since we can easily sample the target distribution and aim to learn the density. We visualize our results in Figure \ref{fig:meteor}.


\subsection{Modelling Invariance Matters}

We also show that our equivariant condition on the manifold flow matters for efficient and accurate training when the target distribution is invariant. In particular, we again consider the sphere under the action of the isotropy group. We try to learn the isotropy invariant density given in Figure \ref{fig:sphere_invariant} and compare the results of our equivariant flow against those of a predefined manifold flow  that does not explicitly model the symmetry \citep{lou2020neural}. While other manifold flow models have been proposed for the sphere \citep{rezende2020normalizing}, NMODE outperforms them~\cite{lou2020neural}, so we use NMODE as a strong baseline. We train for 100 epochs with a learning rate of $0.001$ and a batch size of $200$; our results are shown in Figure \ref{fig:sphere_iso}.

\begin{figure}[H]
    \centering
    \begin{subfigure}[b]{0.26\textwidth}
         \centering
         \includegraphics[width=\textwidth]{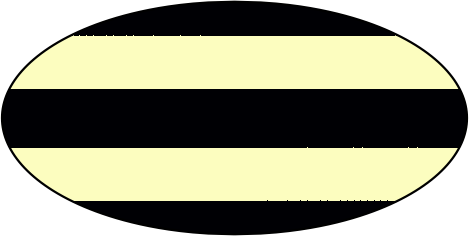}
         \caption{Ground Truth}
         \label{fig:sphere_invariant}
     \end{subfigure}
     \hfill
     \begin{subfigure}[b]{0.26\textwidth}
         \centering
         \includegraphics[width=\textwidth]{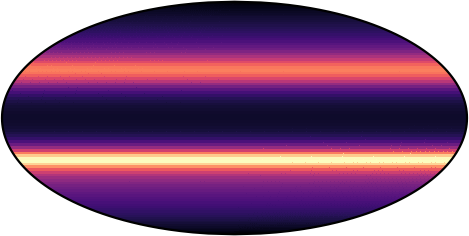}
         \caption{Isotropy Equivariant Flow}
     \end{subfigure}
     \hfill
    \begin{subfigure}[b]{0.26\textwidth}
         \centering
         \includegraphics[width=\textwidth]{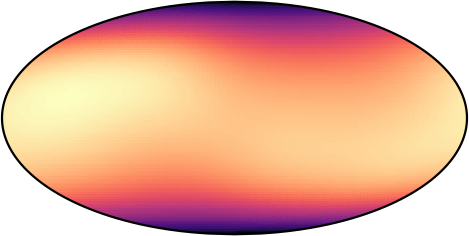}
         \caption{NMODE \cite{lou2020neural}}
     \end{subfigure}
     \caption{We compare the equivariant manifold flow and a regular manifold flow (implemented with NMODE \cite{lou2020neural}) on an invariant dataset. Note that our model is able to accurately capture the ground truth data distribution while NMODE struggles.}
     \label{fig:sphere_iso}
     \vspace{-5pt}
\end{figure}

Despite our equivariant flow having fewer parameters (as both flows have the same width and the equivariant flow has an input dimension of $1$), our model is able to capture the distribution much better than NMODE \cite{lou2020neural}. This is due to the inductive bias of our equivariant model which explicitly leverages the underlying symmetry.

\section{Conclusion}
\label{ref:conclusion}

In this work, we introduce equivariant manifold flows in a fully general context and provide the necessary theory to ensure our construction is principled. We also demonstrate the efficacy of our approach in the context of learning conjugation invariant densities over $SU(2)$ and $SU(3)$, which is an important task for sampling $SU(n)$ lattice gauge theories in quantum field theory. In particular, we show that our method can more naturally adhere to the geometry of the target densities when compared to prior work while being more generally applicable. We also present an application to modelling asteroid impacts and demonstrate the necessity of modelling existing invariances by comparing against a regular manifold flow.

\noindent \textbf{Further considerations.} While our theory and implementations have utility in very general settings, there are still some limitations that could be addressed in future work.
Further research may focus on finding other ways to generate equivariant manifold flows that do not rely on the construction of an invariant potential, and perhaps additionally on showing that such methods are sufficiently expressive to learn over open manifolds. Our models also require a fair bit of tuning to achieve results as strong as we demonstrate. Finally, we note that our theory and learning algorithm are too abstract for us to be sure of the future societal impacts. Still, we advance the field of deep generative models, which is known to have potential for negative impacts through malicious generation of fake images and text. Nevertheless, we do not expect this work to have negative effects in this area, as our applications are not in this domain.




\section*{Acknowledgements}

We would like to thank Facebook AI for funding equipment that made
this work possible. In addition, we thank the National Science Foundation for awarding Prof. Christopher de Sa a grant that helps fund this research effort (NSF IIS-2008102) and for supporting Aaron Lou with a graduate student fellowship. We would also like to acknowledge Jonas K\"{o}hler and Denis Boyda for their useful insights.

\bibliography{neqmode}

\begin{thebibliography}{45}
\providecommand{\natexlab}[1]{#1}
\providecommand{\url}[1]{\texttt{#1}}
\expandafter\ifx\csname urlstyle\endcsname\relax
  \providecommand{\doi}[1]{doi: #1}\else
  \providecommand{\doi}{doi: \begingroup \urlstyle{rm}\Url}\fi

\bibitem[Abel(1824)]{abel1824memoire}
Niels~Henrik Abel.
\newblock \emph{M{\'e}moire sur les {\'e}quations alg{\'e}briques, o{\`u} on
  demontre l'impossibilit{\'e} de la r{\'e}solution de l'{\'e}quation
  g{\'e}n{\'e}rale du cinqui{\`e}me d{\'e}gr{\'e}}.
\newblock 1824.

\bibitem[Bose et~al.(2020)Bose, Smofsky, Liao, Panangaden, and
  Hamilton]{bose2020latent}
Joey Bose, Ariella Smofsky, Renjie Liao, Prakash Panangaden, and Will Hamilton.
\newblock Latent variable modelling with hyperbolic normalizing flows.
\newblock In \emph{Proceedings of the 37th International Conference on Machine
  Learning}, pages 1045--1055, 2020.

\bibitem[Boyda et~al.(2020)Boyda, Kanwar, Racani{\`e}re, Rezende, Albergo,
  Cranmer, Hackett, and Shanahan]{boyda2020sampling}
Denis Boyda, Gurtej Kanwar, S{\'e}bastien Racani{\`e}re, Danilo~Jimenez
  Rezende, Michael~S Albergo, Kyle Cranmer, Daniel~C Hackett, and Phiala~E
  Shanahan.
\newblock Sampling using $ su (n) $ gauge equivariant flows.
\newblock \emph{arXiv preprint arXiv:2008.05456}, 2020.

\bibitem[Bronstein et~al.(2021)Bronstein, Bruna, Cohen, and
  Velivckovi\'c]{Bronstein2021GeometricDL}
Michael~M. Bronstein, Joan Bruna, Taco Cohen, and Petar Velivckovi\'c.
\newblock Geometric deep learning: Grids, groups, graphs, geodesics, and
  gauges.
\newblock \emph{arXiv preprint arXiv:2104.13478}, 2021.

\bibitem[Bump(2004)]{bump2004lie}
Daniel Bump.
\newblock \emph{Lie groups}.
\newblock Springer, 2004.

\bibitem[Chen et~al.(2018)Chen, Rubanova, Bettencourt, and
  Duvenaud]{chen2018neural}
Ricky T.~Q. Chen, Yulia Rubanova, Jesse Bettencourt, and David~K Duvenaud.
\newblock Neural ordinary differential equations.
\newblock In \emph{Advances in Neural Information Processing Systems},
  volume~31, pages 6571--6583, 2018.

\bibitem[Cohen and Welling(2016)]{cohen2016group}
Taco Cohen and Max Welling.
\newblock Group equivariant convolutional networks.
\newblock In \emph{Proceedings of The 33rd International Conference on Machine
  Learning}, pages 2990--2999, 2016.

\bibitem[Cohen et~al.(2019)Cohen, Weiler, Kicanaoglu, and
  Welling]{cohen2019gauge}
Taco Cohen, Maurice Weiler, Berkay Kicanaoglu, and Max Welling.
\newblock Gauge equivariant convolutional networks and the icosahedral {CNN}.
\newblock In \emph{Proceedings of the 36th International Conference on Machine
  Learning}, pages 1321--1330, 2019.

\bibitem[Cohen et~al.(2018)Cohen, Geiger, Köhler, and
  Welling]{cohen2018spherical}
Taco~S. Cohen, Mario Geiger, Jonas Köhler, and Max Welling.
\newblock Spherical {CNN}s.
\newblock In \emph{International Conference on Learning Representations}, 2018.

\bibitem[Dinh et~al.(2017)Dinh, Sohl-Dickstein, and Bengio]{dinh2016density}
Laurent Dinh, Jascha Sohl-Dickstein, and Samy Bengio.
\newblock Density estimation using real nvp.
\newblock In \emph{International Conference on Learning Representations}, 2017.

\bibitem[do~Carmo(1992)]{docarmo}
Manfredo~Perdigão do~Carmo.
\newblock \emph{Riemannian geometry / Manfredo do Carmo ; translated by Francis
  Flaherty.}
\newblock Mathematics. Theory and applications. Birkhäuser, Boston, 1992.
\newblock ISBN 0817634908.

\bibitem[Donnelly(2006)]{Donnelly2006EigenfunctionsOT}
Harold~G. Donnelly.
\newblock Eigenfunctions of the laplacian on compact riemannian manifolds.
\newblock \emph{Asian Journal of Mathematics}, 10:\penalty0 115--126, 2006.

\bibitem[Durkan et~al.(2019)Durkan, Bekasov, Murray, and
  Papamakarios]{Durkan2019NeuralSF}
Conor Durkan, Artur Bekasov, Iain Murray, and George Papamakarios.
\newblock Neural spline flows.
\newblock In \emph{Advances in Neural Information Processing Systems},
  volume~32, 2019.

\bibitem[Earth Impact Database()]{earth_impact_database}
Earth Impact Database.
\newblock Earth impact database, 2011.
\newblock Retrieved from \url{http://passc.net/EarthImpactDatabase}.

\bibitem[Falorsi and Forr{\'e}(2020)]{Falorsi2020NeuralOD}
Luca Falorsi and Patrick Forr{\'e}.
\newblock Neural ordinary differential equations on manifolds.
\newblock \emph{arXiv preprint arXiv:2006.06663}, 2020.

\bibitem[Feiten et~al.(2013)Feiten, Lang, and Hirche]{Feiten2013RigidME}
Wendelin Feiten, Muriel Lang, and Sandra Hirche.
\newblock Rigid motion estimation using mixtures of projected gaussians.
\newblock \emph{Proceedings of the 16th International Conference on Information
  Fusion}, pages 1465--1472, 2013.

\bibitem[Field(1980)]{field1980equivariant}
MJ~Field.
\newblock Equivariant dynamical systems.
\newblock \emph{Transactions of the American Mathematical Society},
  259\penalty0 (1):\penalty0 185--205, 1980.

\bibitem[Finzi et~al.(2020)Finzi, Stanton, Izmailov, and
  Wilson]{finzi2020generalizing}
Marc Finzi, Samuel Stanton, Pavel Izmailov, and Andrew~Gordon Wilson.
\newblock Generalizing convolutional neural networks for equivariance to lie
  groups on arbitrary continuous data.
\newblock In \emph{International Conference on Machine Learning}, pages
  3165--3176. PMLR, 2020.

\bibitem[Gallier and Quaintance(2020)]{Gallier2020DifferentialGA}
Jean Gallier and Jocelyn Quaintance.
\newblock \emph{Differential Geometry and Lie Groups: A Computational
  Perspective}, volume~12.
\newblock Springer, 2020.

\bibitem[Goodfellow et~al.(2014)Goodfellow, Pouget-Abadie, Mirza, Xu,
  Warde-Farley, Ozair, Courville, and Bengio]{Goodfellow2014GenerativeAN}
Ian~J Goodfellow, Jean Pouget-Abadie, Mehdi Mirza, Bing Xu, David Warde-Farley,
  Sherjil Ozair, Aaron Courville, and Yoshua Bengio.
\newblock Generative adversarial nets.
\newblock In \emph{Proceedings of the 27th International Conference on Neural
  Information Processing Systems - Volume 2}, page 2672–2680, 2014.

\bibitem[Grathwohl et~al.(2019)Grathwohl, Chen, Bettencourt, and
  Duvenaud]{grathwohl2018scalable}
Will Grathwohl, Ricky T~Q Chen, Jesse Bettencourt, and David Duvenaud.
\newblock Scalable reversible generative models with free-form continuous
  dynamics.
\newblock In \emph{International Conference on Learning Representations}, 2019.

\bibitem[Hamelryck et~al.(2006)Hamelryck, Kent, and
  Krogh]{Hamelryck2006SamplingRP}
Thomas Hamelryck, John~T Kent, and Anders Krogh.
\newblock Sampling realistic protein conformations using local structural bias.
\newblock \emph{PLoS Computational Biology}, 2\penalty0 (9), 2006.

\bibitem[Kingma and Ba(2015)]{kingma2014adam}
Diederik~P Kingma and Jimmy Ba.
\newblock Adam: {A} method for stochastic optimization.
\newblock In \emph{International Conference on Learning Representations}, 2015.

\bibitem[Kingma and Welling(2014)]{kingma2013auto}
Diederik~P Kingma and Max Welling.
\newblock Auto-encoding variational bayes.
\newblock In \emph{International Conference on Learning Representations}, 2014.

\bibitem[Kobyzev et~al.(2020)Kobyzev, Prince, and Brubaker]{Kobyzev_2020}
Ivan Kobyzev, Simon Prince, and Marcus Brubaker.
\newblock Normalizing flows: An introduction and review of current methods.
\newblock \emph{IEEE Transactions on Pattern Analysis and Machine
  Intelligence}, 2020.

\bibitem[K{\"o}hler et~al.(2020)K{\"o}hler, Klein, and
  Noe]{kohler2020equivariant}
Jonas K{\"o}hler, Leon Klein, and Frank Noe.
\newblock Equivariant flows: Exact likelihood generative learning for symmetric
  densities.
\newblock In \emph{Proceedings of the 37th International Conference on Machine
  Learning}, pages 5361--5370, 2020.

\bibitem[Kondor and Trivedi(2018)]{kondor2018generalization}
Risi Kondor and Shubhendu Trivedi.
\newblock On the generalization of equivariance and convolution in neural
  networks to the action of compact groups.
\newblock In \emph{Proceedings of the 35th International Conference on Machine
  Learning}, pages 2747--2755, 2018.

\bibitem[Lee(2013)]{lee2013introduction}
John~M Lee.
\newblock \emph{Introduction to Smooth Manifolds}.
\newblock Graduate Texts in Mathematics. Springer New York, 2013.

\bibitem[Lou et~al.(2020)Lou, Lim, Katsman, Huang, Jiang, Lim, and
  De~Sa]{lou2020neural}
Aaron Lou, Derek Lim, Isay Katsman, Leo Huang, Qingxuan Jiang, Ser-Nam Lim, and
  Christopher De~Sa.
\newblock Neural manifold ordinary differential equations.
\newblock In \emph{Advances in Neural Information Processing Systems},
  volume~33, pages 17548--17558, 2020.

\bibitem[Mathieu and Nickel(2020)]{mathieu2020riemannian}
Emile Mathieu and Maximilian Nickel.
\newblock Riemannian continuous normalizing flows.
\newblock In \emph{Advances in Neural Information Processing Systems},
  volume~33, pages 2503--2515, 2020.

\bibitem[Meteorite Landings()]{data.world_2017}
Meteorite Landings.
\newblock Meteorite landings dataset, March 2017.
\newblock Retrieved from \url{https://data.world/nasa/meteorite-landings}.

\bibitem[Mezzadri(2007)]{mezzadri2006generate}
Francesco Mezzadri.
\newblock How to generate random matrices from the classical compact groups.
\newblock \emph{Notices of the American Mathematical Society}, 54:\penalty0
  592--604, 2007.

\bibitem[Nagano et~al.(2019)Nagano, Yamaguchi, Fujita, and
  Koyama]{Nagano2019AWN}
Yoshihiro Nagano, Shoichiro Yamaguchi, Yasuhiro Fujita, and Masanori Koyama.
\newblock A wrapped normal distribution on hyperbolic space for gradient-based
  learning.
\newblock In \emph{Proceedings of the 36th International Conference on Machine
  Learning}, pages 4693--4702, 2019.

\bibitem[Papamakarios et~al.(2021)Papamakarios, Nalisnick, Rezende, Mohamed,
  and Lakshminarayanan]{papamakarios2019normalizing}
George Papamakarios, Eric Nalisnick, Danilo~Jimenez Rezende, Shakir Mohamed,
  and Balaji Lakshminarayanan.
\newblock Normalizing flows for probabilistic modeling and inference.
\newblock \emph{Journal of Machine Learning Research}, 22\penalty0
  (57):\penalty0 1--64, 2021.

\bibitem[Paszke et~al.(2017)Paszke, Gross, Chintala, Chanan, Yang, DeVito, Lin,
  Desmaison, Antiga, and Lerer]{Paszke2017AutomaticDI}
Adam Paszke, Sam Gross, Soumith Chintala, Gregory Chanan, Edward Yang, Zach
  DeVito, Zeming Lin, Alban Desmaison, Luca Antiga, and Adam Lerer.
\newblock Automatic differentiation in pytorch.
\newblock In \emph{Neural Information Processing System Autodiff Workshop},
  2017.

\bibitem[Paszke et~al.(2019)Paszke, Gross, Massa, Lerer, Bradbury, Chanan,
  Killeen, Lin, Gimelshein, Antiga, et~al.]{paszke2019pytorch}
Adam Paszke, Sam Gross, Francisco Massa, Adam Lerer, James Bradbury, Gregory
  Chanan, Trevor Killeen, Zeming Lin, Natalia Gimelshein, Luca Antiga, et~al.
\newblock Pytorch: An imperative style, high-performance deep learning library.
\newblock \emph{Advances in neural information processing systems},
  32:\penalty0 8026--8037, 2019.

\bibitem[Rezende et~al.(2019)Rezende, Racani{\`e}re, Higgins, and
  Toth]{rezende2019equivariant}
Danilo~Jimenez Rezende, S{\'e}bastien Racani{\`e}re, Irina Higgins, and Peter
  Toth.
\newblock Equivariant hamiltonian flows.
\newblock \emph{arXiv preprint arXiv:1909.13739}, 2019.

\bibitem[Rezende et~al.(2020)Rezende, Papamakarios, Racaniere, Albergo, Kanwar,
  Shanahan, and Cranmer]{rezende2020normalizing}
Danilo~Jimenez Rezende, George Papamakarios, Sebastien Racaniere, Michael
  Albergo, Gurtej Kanwar, Phiala Shanahan, and Kyle Cranmer.
\newblock Normalizing flows on tori and spheres.
\newblock In \emph{Proceedings of the 37th International Conference on Machine
  Learning}, pages 8083--8092, 2020.

\bibitem[Risken and Frank(1996)]{Risken1984TheFE}
Hannes Risken and Till Frank.
\newblock \emph{The Fokker-Planck Equation: Methods of Solution and
  Applications}.
\newblock Springer Series in Synergetics. Springer Berlin Heidelberg, 1996.

\bibitem[Skopek et~al.(2020)Skopek, Ganea, and
  Bécigneul]{Skopek2020MixedcurvatureVA}
Ondrej Skopek, Octavian-Eugen Ganea, and Gary Bécigneul.
\newblock Mixed-curvature variational autoencoders.
\newblock In \emph{International Conference on Learning Representations}, 2020.

\bibitem[Wang et~al.(2019)Wang, Dang, Hu, Fua, and
  Salzmann]{Wang2019BackpropagationFriendlyE}
Wei Wang, Zheng Dang, Yinlin Hu, Pascal Fua, and Mathieu Salzmann.
\newblock Backpropagation-friendly eigendecomposition.
\newblock In \emph{Advances in Neural Information Processing Systems},
  volume~32, 2019.

\bibitem[Ward and Asphaug(2003)]{wardasteroidimpact}
Steven~N Ward and Erik Asphaug.
\newblock {Asteroid impact tsunami of 2880 March 16}.
\newblock \emph{Geophysical Journal International}, 153\penalty0 (3):\penalty0
  F6--F10, 2003.

\bibitem[Wasserman(1969)]{wasserman1969equivariant}
Arthur~G Wasserman.
\newblock Equivariant differential topology.
\newblock \emph{Topology}, 8\penalty0 (2):\penalty0 127--150, 1969.

\bibitem[Wirnsberger et~al.(2020)Wirnsberger, Ballard, Papamakarios,
  Abercrombie, Racani{\`e}re, Pritzel, Jimenez~Rezende, and
  Blundell]{Wirnsberger2020TargetedFE}
Peter Wirnsberger, Andrew~J Ballard, George Papamakarios, Stuart Abercrombie,
  S{\'e}bastien Racani{\`e}re, Alexander Pritzel, Danilo Jimenez~Rezende, and
  Charles Blundell.
\newblock Targeted free energy estimation via learned mappings.
\newblock \emph{The Journal of Chemical Physics}, 153\penalty0 (14):\penalty0
  144112, 2020.

\bibitem[Zaheer et~al.(2017)Zaheer, Kottur, Ravanbakhsh, Poczos, Salakhutdinov,
  and Smola]{zaheer2017deep}
Manzil Zaheer, Satwik Kottur, Siamak Ravanbakhsh, Barnabas Poczos, Russ~R
  Salakhutdinov, and Alexander~J Smola.
\newblock Deep sets.
\newblock In \emph{Advances in Neural Information Processing Systems},
  volume~30, pages 3391--3401, 2017.

\end{thebibliography}
\bibliographystyle{plainnat}

\newpage

{\LARGE\textbf{Appendix}}

\appendix
\section{Proof of Theorems}\label{appendix:proofs}

In this section, we restate and prove the theorems in Section \ref{sec:theory_main}. These give the theoretical foundations that we use to build our models. Prior work \citep{wasserman1969equivariant, field1980equivariant} addresses some of the results we formalize below.

\subsection{Proof of Theorem 1}

\begin{thmn}[\ref{thm:1}]
Let $(\M, h)$ be a Riemannian manifold and $G$ be its group of isometries (or an isometry subgroup). If $\Phi: \M \to \R$ is a smooth $G$-invariant function, then the following diagram commutes for any $g \in G$: 

\begin{center}
    \begin{tikzcd}
        \M \arrow[r, "L_g"] \arrow[d, "\nabla \Phi"]& \M \arrow[d, "\nabla \Phi"]\\
        T\M \arrow[r, "DL_g"] & T\M
    \end{tikzcd}
\end{center}

or $\nabla_{L_g u} \Phi = D_u L_g(\nabla_u \Phi)$. This is condition is also tight in the sense that it only occurs if $G$ is the group of isometries.
\end{thmn}
\begin{proof}
    We first recall the Riemannian gradient chain rule:
    \begin{equation*}
        \nabla_u (\Phi \circ L_g) = (D_u L_g)^\top (\nabla_{L_g u} \Phi)
    \end{equation*}
    where $(D_u L_g)^\top : T_{L_gu}\M \to T_u\M$ is the ``adjoint" given by \begin{equation*}
        h\left(D_u L_g(v), w\right) = h\left(v, (D_uL_g)^\top (w)\right).
    \end{equation*} 
    Since $L_g$ is an isometry, we also have 
    \begin{equation*}
        h(x, y) = h\left(D_uL_g(x), D_uL_g(y)\right).
    \end{equation*}
    Combining the above two equations gives \begin{equation*}
        h(x, y) = h(D_uL_g(x), D_uL_g(y)) = h\left(x, (D_uL_g)^\top\left(D_uL_g(y)\right)\right), 
    \end{equation*}
    which implies for all $y$, \begin{equation*}
        h\left(x, y - (D_uL_g)^\top(D_uL_g(y))\right) = 0.
    \end{equation*}
    Since $h$ is a Riemannian metric (even pseudo-metric works due to non-degeneracy), we must have that $(D_uL_g)^\top \circ (D_uL_g) = I$. 
    
    To complete the proof, we recall that $\Phi = \Phi \circ L_g$, and this combined with chain rule gives
    \begin{equation*}
        \nabla_u \Phi = \nabla_u (\Phi \circ L_g) = (D_u L_g)^\top (\nabla_{L_g u} \Phi).
    \end{equation*}
    Now applying $D_uL_g$ on both sides gives
    \begin{equation*}
        \nabla_{L_gu} \Phi = D_uL_g \nabla_u \Phi
    \end{equation*}
    which is exactly what we want to show.
    
    We see that this is an ``only if" condition because we must necessarily get that the adjoint is the inverse, which implies that $L_g$ is an isometry.
\end{proof}

\subsection{Proof of Theorem 2}



\begin{thmn}[\ref{thm:2}]
Let $(\M, h)$ be a Riemannian manifold, and $G$ be its isometry group (or one of its subgroups). Let $X$ be any time-dependent vector field on $\mathcal{M}$, and $F_{X, T}$ be the flow of $X$. Then $X$ is an $G$-equivariant vector field if and only if $F_{X,T}$ is a $G$-equivariant flow for any $T\in [0, +\infty)$.
\end{thmn}
\begin{proof}
    $\boldsymbol G$\textbf{-equivariant} $\boldsymbol{X\Rightarrow}$ $\boldsymbol G$\textbf{-equivariant} $\boldsymbol{F_{X,T}}$.
    We invoke the following lemma from \citet[Corollary 9.14]{lee2013introduction}:
    
    \begin{lemma}\label{lem:1}
    Let $F : \M \rightarrow \mathcal{N}$ be a diffeomorphism. If $X$ is a smooth vector field over $\M$ and $\theta$ is the flow of X, then the flow of $F_{*}X$ ($F_*$ is another notation for the differential of $F$) is $\eta_t = F \circ \theta_t \circ F^{-1}$, with domain $N_t = F(M_t)$ for each $t \in \mathbb{R}$.
    \end{lemma}
    Examine $L_g$ and its action on $X$. Since $X$ is $G$-equivariant, we have for any $(x, t)\in \mathcal{M}\times [0, +\infty)$,  \begin{equation*}
        ((L_g)_* X)(x, t) = (D_{L_g^{-1}(x)}L_g)X(L_g^{-1}(x), t) = X(L_g \circ L_g^{-1}(x), t) = X(x, t)
    \end{equation*} 
    so it follows that $(L_g)_* X = X$. Applying the lemma above, we get:
    \[
        F_{(L_g)_*X, T} = L_g \circ F_{X, T} \circ L_g^{-1}
    \]
    
    and, by simplifying, we get that $F_{X, T} \circ L_g = L_g \circ F_{X, T}$, as desired.

    $\boldsymbol G$\textbf{-equivariant} $\boldsymbol{X\Leftarrow}$ $\boldsymbol G$\textbf{-equivariant} $\boldsymbol{F_{X,T}}$. This direction follows from the chain rule. If $F_{X, T}$ is $G$-equivariant, then at all times we have:
    \begin{align*}
        (D_m L_g)\left(X(F_{X, t}(m), t\right) &= (D_m L_g)\left(\frac{d}{dt} F_{X, T}(m)\right)&& \text{(definition)}\\
        &= \frac{d}{dt} (L_g \circ F_{X, T})(m) && \text{(chain rule)}\\
        &= \frac{d}{dt} F_{X, T}(L_gm) && \text{(equivariance)}\\
        &= X(L_g(F_{X, t}(m)), t) && \text{(definition)}
    \end{align*}
    This concludes the proof of the backward direction.
\end{proof}

\subsection{Proof of Theorem 3}

\begin{thmn}[\ref{thm:3}] 
Let $(\M, h)$ be a Riemannian manifold, and $G$ be its isometry group (or one of its subgroups). If $\rho$ is a $G$-invariant density on $\M$, and $f$ is a $G$-equivariant diffeomorphism, then $\rho_f(x)$ is also $G$-invariant.
\end{thmn}
\begin{proof}
We wish to show $\rho_f(x)$ is also $G$-invariant, i.e. $\rho_f(L_g x) = \rho_f(x)$ for all $g \in G, x \in \M$. 

We first recall the definition of $\rho_{f}$:
\begin{equation*}
    \rho_{f}(x)=\rho\left(f^{-1}(x)\right)\left|\det\frac{\partial f^{-1}(x)}{\partial x}\right|=\rho\left(f^{-1}(x)\right)\left|\det J_{f^{-1}}(x)\right|.
\end{equation*}

Since $f \in C^1 (\M,\M)$ is $G$-equivariant, we have $f \circ L_g = L_g \circ f$ for any $g \in G$. Also, since $\rho$ is $G$-invariant, we have $\rho \circ L_g = \rho$. Combining these properties, we see that:

\begin{align*}
    \rho_f (L_g x) &= \rho_f (L_g x) \frac{|\det J_{L_g}(x)|}{|\det J_{L_g}(x)|}=\frac{\rho_{R_{g^{-1}} \circ f}(x)}{|\det J_{L_g}(x)|} && \text{(expanding definition of $\rho_{f}$)}\\
    &=\frac{\rho_{f \circ R_{g^{-1}}}(x)}{|\det J_{L_g}(x)|}= \rho\left((L_g \circ f^{-1})(x)\right) \frac{|\det J_{L_g \circ f^{-1}}(x)|}{|\det J_{L_g}(x)|} && \text{(G-equivariance of }f) \\
    &= (\rho \circ L_g \circ f^{-1})(x) \frac{|\det J_{L_g} (f^{-1}(x)) J_{f^{-1}}(x)|}{|\det J_{L_g}(x)|} && \text{(expanding Jacobian)} \\
    &= (\rho \circ f^{-1})(x) \frac{|\det J_{L_g} (f^{-1}(x))| |\det J_{f^{-1}}(x)|}{|\det J_{L_g}(x)|}  && \text{(G-invariance of }\rho)\\
    &= \rho(f^{-1}(x)) |\det J_{f^{-1}}(x)| \cdot \frac{|\det J_{L_g} (f^{-1}(x))|}{|\det J_{L_g}(x)|} && \text{(rearrangement)}\\
    &= \rho_f(x) \cdot \frac{|\det J_{L_g} (f^{-1}(x))|}{|\det J_{L_g}(x)|} && \text{(expanding definition of $\rho_{f}$)}
\end{align*}

Now note that $G$ is contained in the isometry group, and thus $L_g$ is an isometry. This means $|\det J_{L_g}(x)|=1$ for any $x\in \mathcal{M}$, so the right-hand side above is simply $\rho_{f}(x)$, which proves the theorem.
\end{proof}

\subsection{Proof of Theorem 4}

\begin{thmn}[\ref{thm:4}]
Let $(\M, h)$ be a closed Riemannian manifold. Let $\pi$ be a smooth, non-vanishing distribution over $\M$, which will act as our target distribution. Let $\rho_t$ be a distribution over said manifold parameterized by a real time variable $t$, with $\rho_0$ acting as the initial distribution. Let $D_{KL}(\rho_t||\pi)$ denote the Kullback–Leibler divergence between distributions $\rho_t$ and $\pi$. If we choose a $g: \M \rightarrow \mathbb{R}$ such that
\[
g(x) = \log \left( \frac{\pi(x)}{\rho_t(x)} \right),
\]
and if $\rho_t$ evolves with $t$ as the distribution of a flow according to $g$, it follows that
\[
\frac{\partial}{\partial t} D_{KL} (\rho_t \| \pi) = - \int_{\mathcal{M}} \rho_t \exp(g) \| \nabla g \|^2\: dx
= - \int_\mathcal{M} \pi \| \nabla g \|^2\: dx
\]

implying convergence of $\rho_t$ to $\pi$ in $KL$. Moreover, the exact diffeomorphism that takes us from $\rho_0 \rightarrow \pi$ is as follows. Given some initial point $x \in \mathcal{M}$, let $u(t)$ be the solution to the initial value problem given by:
\[
\frac{du(t)}{dt} = \nabla g(t), \qquad u(0) = x
\]
The desired diffeomorphism maps $x$ to $\lim_{t \rightarrow \infty} u(t)$.
\end{thmn}
\begin{proof}
\textbf{1) Derivative of $D_{KL}(\rho_t || \pi)$.} We start by noting the following: by the Fokker-Planck equation, $\rho_t$ evolving as a flow according to $g$ is equivalent to
\[
\frac{\partial \rho_t}{\partial t} = \nabla \cdot (\rho_t \nabla g).
\]

Please observe that since $\rho_t$ is defined as being a solution to the Fokker-Planck equation \cite{Risken1984TheFE}, $\rho_t$ will be a family of densities. In particular, the Fokker-Planck equation describes the time evolution of a probability density function.

Keeping Fokker-Planck in mind, we obtain the following expression for the time derivative of $D_{KL}(\rho_t || \pi)$:
\begin{equation*}
\begin{split}
    \frac{\partial}{\partial t} D_{KL}(\rho_t || \pi) &= \int \frac{\pi}{\rho_t} \frac{\partial \rho_t}{\partial t} \: dx \\
    & = \int \frac{\pi}{\rho_t} \nabla \cdot (\rho_t \nabla g) \: dx \\
    & = \int \left(\nabla \cdot 
    \left(\frac{\pi}{\rho_t} (\rho_t \nabla g) \right) - (\rho_t \nabla g) \cdot \nabla \cdot \frac{\pi}{\rho_t}\right) \: dx \\
    & = \int \left(\nabla \cdot (\pi \nabla g) - (\rho_t \nabla g) \cdot \nabla \frac{\pi}{\rho_t} \right) \: dx \\
    & = -\int (\rho_t \nabla g) \cdot \nabla \frac{\pi}{\rho_t} \: dx,
\end{split}
\end{equation*}

where the final equality follows from the divergence theorem, since the integral of the divergence over a closed manifold is $0$. Now if we choose $g$ such that:
\[
g(x) = \log \left(\frac{\pi(x)}{\rho_t(x)} \right).
\]
Then we have:
\begin{equation*}
\begin{split}
    \frac{\partial}{\partial t} D_{KL}(\rho_t || \pi) &= - \int (\rho_t \nabla g) \cdot \nabla \exp(g) \: dx \\
    & = - \int \rho_t \exp(g) ||\nabla g||^2 \: dx,
\end{split}
\end{equation*}

\textbf{2) Proof of convergence.} Consider:

\[
\frac{\partial \rho_t}{\partial t} = \nabla \cdot (\rho_t \nabla g)
\]

where $g$ is defined as above. Note by standard existence and uniqueness results for differential equations on manifolds (for example, see \citet{docarmo}) we have the existence of a solution, $\rho_t$ for all time $t > 0$, to this differential equation with initial value $\rho_0$.

Now note $g$, expressed as a function of $\rho_t$, is an invariant potential, the flow of which maps $\rho_0$ to $\lim_{t \rightarrow \infty} \rho_t$. By the result above, we know the right-hand-side of the equation:

\begin{equation*}
\begin{split}
    \frac{\partial}{\partial t} D_{KL}(\rho_t || \pi)
    & = - \int \rho_t \exp(g) ||\nabla g||^2 \: dx,
\end{split}
\end{equation*}

must approach $0$ (since the $KL$-divergence cannot continue decreasing at any constant rate, as it must be non-negative). The only way the right-hand-side can be $0$ is when $\nabla g = 0$, which can occur only when $\rho_t = \pi$. This concludes the proof of convergence of $\rho_t \rightarrow \pi$ in $KL$. \\

\textbf{3) Showing diffeomorphism $\rho_0 \rightarrow \pi$ is well-defined.} The exact diffeomorphism from $\rho \rightarrow \pi$ is as follows. Given some initial point $x \in \text{supp}(\rho)$, let $u(t)$ be the solution to the initial value problem given by:
\[
\frac{du(t)}{dt} = \nabla g(t), \qquad u(0) = x
\]
$g$ is defined as before. Note $u(t)$ exists and is unique by standard differential equation uniqueness and existence results \cite{docarmo}. We claim the desired diffeomorphism maps $x$ to $\lim_{t \rightarrow \infty} u(t)$. All that remains is to show (a) convergence to a smooth function at the limit and (b) that equivariance of the diffeomorphism does not break at the limit. We begin by showing this for $\pi$ uniform and finish the proof by extending to $\pi$ general. \\

\textbf{$\pi$ uniform.} For simplicity, we first consider the case where $\pi$ is the uniform (Haar) measure. In this case, the differential equation that $\rho$ obeys reduces to the heat equation, namely:

\[
\frac{\partial \rho}{\partial t} = \Delta \rho
\]

(a) Please note the following: an important fact that makes harmonic analysis on compact manifolds possible is that the spectrum of the Laplacian on any compact manifold must be discrete, i.e. its eigenvalues are countable and tend to infinity \cite{Donnelly2006EigenfunctionsOT}. Also, its eigenvectors must be smooth (intuitively this says harmonic analysis is ``nice" on manifolds in the same way that Fourier analysis is nice on the unit circle).

Note also that the Laplacian is Hermitian and negative semidefinite, and moreover that the only eigenfunction for eigenvalue $0$ is the constant vector.

Both facts above imply the solution to the above differential equation will just be the sum of several exponentially decaying (in $t$) terms and a constant term, given by the harmonic expansion of $\rho_0$.

From here, it follows that the $L^2$ distance between $\rho_t$ and the constant potential is just the sum of squares of the coefficients in front of those terms (this is simply the manifold analog of Parseval's theorem). However, all of those terms are decaying exponentially, so it follows that $\rho_t$ converges in the $L^2$ norm to the constant potential\footnote{Please note that if we wanted some other type of convergence, e.g. pointwise convergence, we could get this as well using a similar argument, by analyzing the decay properties of the eigenvalues/eigenvectors of the Laplacian.}.

(b) Additionally, note that if the initialization $\rho_0$ is $G$-invariant, then it is fairly easy to see that all the terms in its harmonic expansion must also be $G$-invariant. As a result, $\rho$ must be $G$-invariant at all times, and must remain $G$-invariant in the limit. Similarly, its flow must be $G$-equivariant.

\textbf{$\pi$ general.} We have shown the desired properties for the case of $\pi$ uniform. However, the general case is entirely analogous, as the modified operator (involving $\pi$) has all the same relevant properties as the Laplacian (it is just generally better known that the Laplacian has these properties).

\end{proof}

\subsection{Conjugation by \texorpdfstring{$SU(n)$}{SU(n)} is an Isometry}\label{appendix:conjugation_isometry}

We now prove a lemma that shows that the group action of conjugation by $SU(n)$ is an isometry subgroup. This implies that Theorems \ref{thm:1} through \ref{thm:3} above can be specialized to the setting of $SU(n)$. 

\begin{lemma}
Let $G$ be the group action of conjugation by $SU(n)$, and let each $L_g$ represent the corresponding action of conjugation by $g \in SU(n)$. Then $G$ is an isometry subgroup.
\end{lemma}
\begin{proof}
We first show that the matrix conjugation action of $SU(n)$ is unitary. For $R, X \in SU(n)$, note that the action of conjugation is given by $\mrm{vec}(R X R^{-1}) = (R^{-T} \otimes R) \mrm{vec}(X)$. We have that $R^{-T} \otimes R$ is unitary because:
\begin{align*}
    &(R^{-T} \otimes R)^* (R^{-T} \otimes R) \\
    &  = (\overline{R^{-1}} \otimes R^*) (R^{-T} \otimes R) && \text{(conjugate transposes distribute over $\otimes$)}\\    
    & = (\overline{R^{-1}} R^{-T}) \otimes (R^{*} R) && \text{(mixed-product property of $\otimes$)}\\
    & = (R^T R^{-T}) \otimes (I) = (I) \otimes (I) = I_{n^2 \times n^2} && \text{(simplification)}
\end{align*}

Now choose an orthonormal frame $X_1, \ldots, X_m$ of $T\M$. Note that $T\M$ locally consists of $SU(n)$ shifts of the algebra, which itself consists of traceless skew-Hermitian matrices \cite{Gallier2020DifferentialGA}. We show $G$ is an isometry subgroup by noting that when it acts on the frame, the resulting frame is orthonormal. Let $g \in G$, and consider the result of action of $g$ on the frame, namely $L_gX_1, \ldots, L_gX_m$. Then we have:
\[
(L_gX_i)^* (L_gX_j) = X^*_i R^*_g L_g X_j = X^*_i X_j.
\]
Note for $i \neq j$, we have $X^*_i X_j = 0$ and for $i = j$ we see $X^*_i X_i = 1$. Hence the resulting frame is orthonormal and $G$ is an isometry subgroup.
\end{proof}

\section{Manifold Details for the Special Unitary Group \texorpdfstring{$SU(n)$}{SU(n)}}\label{appendix:sun}

In this section, we give a basic introduction to the special unitary group $SU(n)$ and relevant properties.

\textbf{Definition.} The special unitary group $SU(n)$ consists of all $n$-by-$n$ unitary matrices $U$ (i.e. $U^{*}U=UU^{*}=1$ for $U^{*}$ the conjugate transpose of $U$) that have determinant $\det(U)=1$.

Note that $SU(n)$ is a smooth manifold; in particular, it has Lie structure \cite{Gallier2020DifferentialGA}. Moreover, the tangent space at the identity (i.e. the Lie algebra) consists of traceless skew-Hermitian matrices \cite{Gallier2020DifferentialGA}. The Riemannian metric is $\mathrm{tr}(A^\top B)$.

\subsection{Haar Measure on \texorpdfstring{$SU(n)$}{SU(n)}}

\textbf{Haar Measure.} Haar measures are generic constructs of measures on topological groups $G$ that are invariant under group operation. For example, the Lie group $G=SU(n)$ has Haar measure $\mu_{H}:SU(n)\rightarrow\mathbb{R}$, which is defined as the unique measure such that for any $U\in SU(n)$, we have \begin{equation*}
    \mu_{H}(VU)=\mu_{H}(UW)=\mu_{H}(U)
\end{equation*}
for all $V, W\in SU(n)$ and $\mu_H(G) = 1$.

A topological group $G$ together with its unique Haar measure defines a probability space on the group. This gives one natural way of defining probability distributions on the group, explaining its importance in our construction of probability distributions on Lie groups, specifically $SU(n)$.

To make the above Haar measure definition more concrete, we note from \citet[Proposition 18.4]{bump2004lie} that we can transform an integral over $SU(n)$ with respect to the Haar measure into integrating over the corresponding diagonal matrices under eigendecomposition:  \begin{equation*}
    \int\limits_{SU(n)}fd\mu_{H}=\frac{1}{n!}\int\limits_{T}f(\mathrm{diag}(\lambda_{1}, \ldots, \lambda_{n}))\prod\limits_{i<j}|\lambda_{i}-\lambda_{j}|d\lambda.
\end{equation*}
Thus, we can think of the Haar measure as inducing the change of variables with volume element \begin{equation*}
    \mrm{Haar}(x) = \prod_{i < j} |\lambda_i(x) - \lambda_j(x)|^2.
\end{equation*}
To sample uniformly from the Haar measure, we just need to ensure that we are sampling each $x\in SU(n)$ with probability proportional to $\text{Haar}(x)$.

\textbf{Sampling from the Haar Prior.} We use Algorithm~\ref{alg:haar_sample} \citep{mezzadri2006generate} for generating a sample uniformly from the Haar prior on $SU(n)$:

\begin{algorithm}[H]
    Sample $Z\in \mathbb{C}^{n\times n}$ where each entry $Z_{ij} =   Z_{ij}^{(1)} + iZ_{ij}^{(2)}$ for independent random variables $Z_{ij}^{(1)}, Z_{ij}^{(2)} \sim \mathcal{N}(0,1/2)$.
    
    Let $Z=QR$ be the QR Factorization of $Z$.
    
    Let $\Lambda=\mathrm{diag}(\frac{R_{11}}{|R_{11}|}, \ldots, \frac{R_{nn}}{|R_{nn}|})$.
    
    Output $Q'=Q\Lambda$ as distributed with Haar measure.
    
    \caption{Sampling from the Haar Prior on $SU(n)$}
    \label{alg:haar_sample}
\end{algorithm}

\subsection{Eigendecomposition on \texorpdfstring{$SU(n)$}{SU(n)}}


One main step in the invariant potential computation for $SU(n)$ is to derive formulas for the eigendecomposition of $U\in SU(n)$ as well as formulas for double differentiation through the eigendecomposition (recall that we must differentiate the $SU(n)$-invariant potential $\Phi$ to get $SU(n)$-equivariant vector field $\nabla \Phi$ and another time to produce gradients to optimize this). During the initial submission of our paper, a general implementation of this for complex matrices did not exist. Furthermore, while various specialized numerical techniques have been developed \cite{Wang2019BackpropagationFriendlyE} to perform this computation, the implementation of these was unnecessary for our test cases of $n=2, 3$. Instead, we derived explicit formulas for the eigenvalues based on finding roots of the characteristic polynomials (given by root formulas for quadratic/cubic equations). Note that this procedure does not scale to higher dimensions since there does not exist a closed form solution for $n > 4$ \cite{abel1824memoire}. However, concurrently released versions of PyTorch \cite{paszke2019pytorch} introduced twice differentiable complex eigendecomposition, allowing one to easily extend our methods to higher dimensions.

\subsubsection{Explicit Formula for \texorpdfstring{$SU(2)$}{SU(2)}}

We now derive an explicit eigenvalue formula for the $U\in SU(2)$ case. Let us denote $U=\begin{bmatrix}
    a+bi & -c+di \\
    c+di & a-bi
\end{bmatrix}$ for $a,b,c,d \in \mathbb{R}$ such that $a^2+b^2+c^2+d^2=1$ as an element of $SU(2)$; then the characteristic polynomial of this matrix is given by \begin{equation*}
    \det(\lambda I-U)=(\lambda-(a+bi))(\lambda-(a-bi))+(c+di)(c-di)=(a-\lambda)^{2}+b^{2}+c^{2}+d^{2}=\lambda^{2}-2a\lambda+1
\end{equation*}
and thus its eigenvalues are given by \begin{equation*}
    \lambda_{1}=a+i\sqrt{1-a^{2}}=a+i\sqrt{b^{2}+c^{2}+d^{2}}
\end{equation*}
\begin{equation*}
    \lambda_{2}=a-i\sqrt{1-a^{2}}=a-i\sqrt{b^{2}+c^{2}+d^{2}}
\end{equation*}

\textbf{Remark.} We note that there is a natural isomorphism $\phi:S^{3}\rightarrow SU(2)$, given by \begin{equation*}
    \phi(a, b, c, d)=\begin{bmatrix}
        a+bi & -c+di \\
        c+di & a-bi
    \end{bmatrix}
\end{equation*}

We can exploit this isomorphism by learning a flow over $S^3$ with a regular manifold flow like NMODE \cite{lou2020neural} and mapping it to a flow over $SU(2)$. This is also an acceptable way to obtain stable density learning over $SU(2)$.

\subsubsection{Explicit Formula for \texorpdfstring{$SU(3)$}{SU(3)}}

We now derive an explicit eigenvalue formula for the $U\in SU(3)$ case. For the case of $U\in SU(3)$, we can compute the characteristic polynomial as \begin{align*}
    \det(\lambda I-U)&=\det\left(\begin{bmatrix}
        \lambda-U_{11} & -U_{12} & -U_{13} \\
        -U_{21} & \lambda-U_{22} & -U_{23} \\
        -U_{31} & -U_{32} & \lambda-U_{33}
    \end{bmatrix}\right) \\
    &=\lambda^{3}+c_{2}\lambda^{2}+c_{1}\lambda+c_{0}
\end{align*}
where \begin{equation*}
    c_{2}=-(U_{11}+U_{22}+U_{33})
\end{equation*}
\begin{equation*}
    c_{1}=U_{11}U_{22}+U_{22}U_{33}+U_{33}U_{11}-U_{12}U_{21}-U_{23}U_{32}-U_{13}U_{31}
\end{equation*}
\begin{equation*}
    c_{0}=-(U_{12}U_{23}U_{31}+U_{13}U_{21}U_{32}+U_{11}U_{22}U_{33}-U_{12}U_{21}U_{33}-U_{13}U_{31}U_{22}-U_{23}U_{32}U_{11})
\end{equation*}

Now to solve the equation \begin{equation*}
    \lambda^{3}+c_{2}\lambda^{2}+c_{1}\lambda+c_{0}=0
\end{equation*}
we first transform it into a depressed cubic \begin{equation*}
    t^{3}+pt+q=0
\end{equation*}
where we make the transformation \begin{equation*}
    t=x+\frac{c_{2}}{3}
\end{equation*}
\begin{equation*}
    p=\frac{3c_{1}-c_{2}^{2}}{3}
\end{equation*}
\begin{equation*}
    q=\frac{2c_{2}^{3}-9c_{2}c_{1}+27c_{0}}{27}
\end{equation*}
Now from Cardano's formula, we have the cubic roots of the depressed cubic given by \begin{equation*}
    \lambda_{1, 2,3}=\sqrt[3]{-\frac{q}{2}+\sqrt{\frac{q^{2}}{4}+\frac{p^{3}}{27}}}+\sqrt[3]{-\frac{q}{2}-\sqrt{\frac{q^{2}}{4}+\frac{p^{3}}{27}}}
\end{equation*}
where the two cubic roots in the above equation are picked such that they multiply to $-\frac{p}{3}$.

\section{Experimental Details for Learning Equivariant Flows on \texorpdfstring{$SU(n)$}{SU(n)}}

This section presents some additional details regarding the experiments that learn invariant densities on $SU(n)$ in Section \ref{sec:experiments}.

For the evaluation, we found that ESS (effective sample size) was not a good metric to compare learned densities in this context. In particular, we noticed that several degenerate (mode collapsed) densities were able to attain near perfect ESS while completely failing on matching the target distribution geometry. Given that \citet{boyda2020sampling} did not release code and reported ESS only for certain test cases, we decided to exclude ESS as a metric from our paper and instead relied directly on distribution geometry visualization.

\subsection{Training Details}\label{appendix:training}

Our DeepSet network \citep{zaheer2017deep} consists of a feature extractor and regressor. The feature extractor is a $1$-layer tanh network with $32$ hidden channels. We concatenate the time component to the sum component of the feature extractor before feeding the resulting $33$ size tensor into a $1$-layer tanh regressor network.

To train our flows, we minimize the KL divergence between our model distribution and the target distribution \citep{papamakarios2019normalizing}, as is done in \citet{boyda2020sampling}. In a training iteration, we draw a batch of samples uniformly from $SU(n)$, map them through our flow, and compute the gradients with respect to the batch KL divergence between our model probabilities and the target density probabilities. We use the Adam stochastic optimizer for gradient-based optimization \citep{kingma2014adam}. The graph shown in Figure \ref{fig:sun_densities} was trained for $300$ iterations with a batch size of $8192$ and weight decay setting of $0.01$; the starting learning rate for Adam was $0.01$, and a multi-step learning rate schedule that decreased the learning rate by a factor of $10$ every $100$ epochs was used. We use PyTorch to implement our models and run experiments~\cite{paszke2019pytorch}. Experiments are run on one CPU and/or GPU at a time, where we use one NVIDIA RTX 2080Ti GPU with 11 GB of GPU RAM.

\definecolor{mydarkblue}{rgb}{0,0.08,0.65}

We note that during our implementation, there are specific parts of the code that involved careful tuning for effective training. Specifically, we perturbed the results of certain functions and gradients by small constants to ensure numerical stability of the training process. We also spent some time tuning the learning rate and some ODE settings. More details can be found in the accompanying  \textcolor{mydarkblue}{\href{https://github.com/CUAI/Equivariant-Manifold-Flows}{Github code}}.

\subsection{Conjugation-Invariant Target Distributions}

\citet{boyda2020sampling} defined a family of matrix-conjugation-invariant densities on $SU(n)$ as:
\[
p_{toy}(U) =  \frac{1}{Z} e^{\frac{\beta}{n} \text{Re}\:\text{tr}\left( \sum_k c_k U^k \right)} , \]
which is parameterized by scalars $c_k$ and $\beta$. The normalizing constant $Z$ is chosen to ensure that $p_{toy}$ is a valid probability density with respect to the Haar measure.

More specifically, the experiments of \citet{boyda2020sampling} focus on learning to sample from the distribution with the above density with three components, in the following form: \[
p_{toy}(U) =  \frac{1}{Z} e^{\frac{\beta}{n} \text{Re}\:\text{tr}\left(c_{1} U+c_{2}U^{2}+c_{3}U^{3} \right)}\]
We tested on three instances of the density, also used in \citet{boyda2020sampling}: \begin{table}[H]
    \centering
    \begin{tabular}{c c c c c}
        \hline
        set $i$ & $c_{1}$ & $c_{2}$ & $c_{3}$ & $\beta$ \\
        \hline
        1 & 0.98 & -0.63 & -0.21 & 9 \\
        2 & 0.17 & -0.65 & 1.22 & 9 \\
        3 & 1 & 0 & 0 & 9 \\
        \hline
    \end{tabular}
    \caption{Sets of parameters $c_{1}, c_{2}, c_{3}$ and $\beta$ used in the $SU(2)$ and $SU(3)$ experiments}
    \label{tab:parameter_target_distribution}
\end{table}

Note that the rows of Figure \ref{fig:sun_densities} correspond to coefficient sets $3,2,1$, given in order from top to bottom.

\subsubsection{Case for \texorpdfstring{$SU(2)$}{SU(2)}} \label{appendix:target}

In the case of $n=2$, we can represent the eigenvalues of a matrix $U \in SU(2)$ in the form $e^{i\theta}, e^{-i\theta}$ for some angle $\theta\in [0, \pi]$. We then have $\mrm{tr}(U) = e^{i\theta} + e^{-i\theta} = 2\cos(\theta)$, so above density takes the form:
\[
p_{toy}(U) = \frac{1}{Z} e^{c_1\beta \cos \theta}  \cdot e^{c_2\beta \cos (2\theta)}  \cdot e^{c_3\beta \cos (3\theta)}.
\]

\subsubsection{Case for \texorpdfstring{$SU(3)$}{SU(3)}} \label{appendix:target_su3}

In the case of $n=3$, we can represent the eigenvalues of $U\in SU(3)$ in the form $e^{i\theta_{1}}, e^{i\theta_{2}}, e^{i(-\theta_{1}-\theta_{2})}$. Thus, we have \begin{equation*}
    \text{Re}\:\text{tr}(U)=\frac{1}{3}\left(\cos(\theta_{1})+\cos(\theta_{2})+\cos(-\theta_{1}-\theta_{2})\right)
\end{equation*}
and thus \begin{align*}
    p_{toy}(U)&=\frac{1}{Z}e^{\frac{c_{1}\beta}{3}\left(\cos(\theta_{1})+\cos(\theta_{2})+\cos(-\theta_{1}-\theta_{2})\right)} \\
    &\cdot e^{\frac{c_{2}\beta}{3}\left(\cos(2\theta_{1})+\cos(2\theta_{2})+\cos(-2\theta_{1}-2\theta_{2})\right)} \\
    &\cdot e^{\frac{c_{3}\beta}{3}\left(\cos(3\theta_{1})+\cos(3\theta_{2})+\cos(-3\theta_{1}-3\theta_{2})\right)} 
\end{align*}

\section{Learning Continuous Normalizing Flows over Manifolds with Boundary}

\textbf{Motivation.} Recall that learning a continuous normalizing flow over a manifold with boundary is not principled, and is rather numerically unstable, since probability mass can ``flow out" on the boundary. In particular we noted in Section \ref{sec:intro} that this was a major problem for the quotient manifold approach to learning invariant densities, since the quotient frequently has a nonempty boundary.

\textbf{Our Approach.} Our method enables learning flows over manifolds with boundary. One need only represent the manifold with boundary as a quotient of a larger manifold without boundary and learn with an invariant potential function that ensures the density descends smoothly from the larger manifold without boundary to the manifold with boundary.

\textbf{Example.} For instance, one can use our method to construct a flow over an interval. Notice that we can view an interval $I = [0,1]$ as a manifold with boundary. The boundary consists of the two endpoints, $\{0,1\}$. To use our method to learn a flow over this interval, we need only represent $[0,1]$ as the quotient of $S^2$ by the isotropy group at the north pole, then apply the flow construction described in Section \ref{sec:sphere_model}. The learned density assigns the same value to all points at the same latitude: clearly, this descends to a density over $[0,1]$ by taking one representative point from each latitude circle. Notice that this works more generally: we can represent various manifolds with boundary as quotients of larger manifolds by isotropy groups. In particular, one can imagine using this method to replace neural spline flows \citep{Durkan2019NeuralSF}, which carefully constructs noncontinuous normalizing flows over intervals.





\end{document}